%% file: main.tex
\documentclass[10pt]{article} 

\usepackage[preprint]{rlj} 

\usepackage{amssymb}            
\usepackage{mathtools}          
\usepackage{mathrsfs}           
\usepackage{graphicx}           
\usepackage{subcaption}         
\usepackage[space]{grffile}     
\usepackage{url}                
\usepackage{lipsum}             

\usepackage{amsthm}
\usepackage{wrapfig}
\usepackage{algorithm}
\usepackage{algcompatible}
\usepackage{enumitem}

\newtheorem{definition}{Definition}
\newtheorem{remark}{Remark}
\newtheorem{theorem}{Theorem}
\newtheorem{corollary}[theorem]{Corollary}
\newtheorem{lemma}[theorem]{Lemma}

\title{Is Bellman Equation Enough for Learning Control?}

\setrunningtitle{Insufficiency of Bellman Equation}

\author{Haoxiang You\textsuperscript{1}, Lekan Molu\textsuperscript{2}, Ian Abraham\textsuperscript{1}}

\emails{\{haoxiang.you,ian.abraham\}@yale.edu, \ lekanmolu@microsoft.com}

\affiliations{
$^{1}$\textbf{Department of Mechanical Engineering, Yale University}\\
$^{2}$\textbf{Microsoft Research}
}


\begin{document}


\maketitle  

\input{sections/abstract}
\input{sections/introduction}

\input{sections/background}
\input{sections/general_sol}
\input{sections/canonical_failure}
\input{sections/potential_solutions}
\input{sections/related_work}
\input{sections/conclusion}



\bibliography{reference}
\bibliographystyle{rlj}

\input{sections/addtional_proof}
\input{sections/tabular_setting}
\input{sections/toy_example_with_different_initialization}
\input{sections/experiment_setup}

\end{document}

%% file: sections/abstract.tex
\begin{abstract}
The Bellman equation and its continuous-time counterpart, the Hamilton-Jacobi-Bellman (HJB) equation, serve as necessary conditions for optimality in reinforcement learning and optimal control. 
While the value function is known to be the unique solution to the Bellman equation in tabular settings, we demonstrate that this uniqueness \textbf{fails to hold} in continuous state spaces. 
Specifically, for linear dynamical systems, we prove the Bellman equation admits at least $\binom{2n}{n}$ solutions, where $n$ is the state dimension. 
Crucially, \textbf{only one} of these solutions yields both an optimal policy and a stable closed-loop system. 
We then demonstrate a common failure mode in value-based methods: convergence to unstable solutions due to the exponential imbalance between admissible and inadmissible solutions. 
Finally, we introduce a positive-definite neural architecture that guarantees convergence to the stable solution by construction to address this issue.
\end{abstract}

%% file: sections/introduction.tex
\section{Introduction}
\label{sec: Intro}
Reinforcement learning (RL) methods are broadly categorized into policy-based approaches~\citep{williams1992simple, schulman15, schulman2017proximal}, which directly optimize policy parameters using gradient estimates, and value-based approaches~\citep{watkins1992q, mnih2015human, lillicrap2015continuous}, which indirectly learn policies by solving optimality conditions encoded in Bellman equations.
While policy-based methods excel in high-dimensional dynamical system, value-based approaches benefit from sample efficiency through offline learning.
In continuous-time domains, the Hamilton–Jacobi–Bellman (HJB) equation serves as the counterpart to the discrete-time Bellman equation. HJB-based RL enables control at arbitrarily high frequencies~\citep{doya2000reinforcement} and naturally incorporates physical priors for robust policy synthesis~\citep{lutter2021robust}. 
Despite these advantages, practical applications remain largely limited to low-dimensional systems~\citep{nakamura2021adaptive,lutter2023robust}, due to fundamental challenges in solving high-dimensional HJB equations.

\emph{
    Two critical issues hinder the extension of HJB methods to complex systems:
    }
    
\textbf{Solution Existence:} 
The value function (optimal performance index) may lack smoothness or even continuity~\citep{bardi1997optimal, clarke2008nonsmooth}. Common remedies include viscosity solutions~\citep{crandall1983viscosity, shilova2024learning} that generalize differentiability requirements, or stochastic regularization through noise injection~\citep{fleming2006controlled, tassa2007least}.
    
\textbf{Solution Uniqueness:}
Since the Bellman equation is only a necessary condition for optimality, it may admit general solutions beyond the value function.
Early work~\citep{bellman1957dynamic, blackwell1965discounted, sutton1988learning} has shown, value function is indeed the unique solution to Bellman equation for the finite state and action space.
Yet, the uniqueness of general solutions in continuous state space is overlooked in the reinforcement learning community. We show that the Bellman equation admits multiple solutions in continuous state spaces, making the performance of value-based approaches highly dependent on model initialization. This non-uniqueness partially explains why value-based approaches are sensitive to hyperparameter choices and initialization~\citep{ceron2024consistency}.

In this paper, we formally analyze solution spaces for both discrete-time Bellman and continuous-time HJB equations in continuous state and action space. 
For linear dynamical systems, we prove the Bellman equation admits at least $\binom{2n}{n}$ solutions -- a number growing exponentially($\sim 4^n/\sqrt{\pi n}$ asymptotically) with state dimension $n$. 
Crucially, \textbf{only one} solution leads to a stable closed-loop system. 
This exponential disparity between admissible solutions and the unique stabilizing optimum fundamentally obstructs value-based learning approaches from synthesizing effective controllers. 
We characterize this inherent challenge in continuous spaces -- the need to identify the optimal solution among exponentially many alternatives -- as \textit{the curse of dimensionality in solution space}, complementing the well-known \textit{curse of dimensionality in computation} from tabular setting.

Our key contributions are:
\begin{enumerate}[itemsep=-2pt]
    \item A constructive method to derive general solutions of Bellman/HJB equations for linear systems, with theoretical characterization of their relationship to closed-loop stability via spectral analysis.
    \item Identification of a canonical failure mode in practice: value-based methods converge to unstable general solutions.
    \item A positive-definite neural architecture that provably constrains learning to the stable solution subspace, demonstrating its effectiveness beyond linear dynamics.
\end{enumerate}

The rest of the paper is organized as following: Section~\ref{sec: background} covers the background on problem setting, Bellman equation, and learning methods for solving them. Section~\ref{sec: general solutions} focus on the general solutions of Bellman equation and how they affect control. Section~\ref{sec: canonical failure} present a canonical failure mode in value-based learning approach. Section~\ref{sec:potential solutions} discuss techniques to handle this failure. Finally, Section~\ref{sec: related work} highlight the difference between this paper and existing work.

%% file: sections/background.tex
\section{Background}
\label{sec: background}
This section presents the optimal control framework, introduces both continuous-time and discrete-time Bellman equations, and reviews contemporary learning methods for their solution.
\subsection{Problem Setting}
In this paper, we consider both continuous-time system in form of  
\begin{equation}
    \dot{\mathbf{x}}(t) = f(\mathbf{x}(t), \ \text{and} \,\, \mathbf{x}_{k+1} = f(\mathbf{x_k}, \mathbf{u}_k) \label{eq: continuous-dyn}
\end{equation}
in discrete-time.
Here, $\mathbf{x} \in \mathcal{X} \subseteq \mathbb{R}^n$ is the state and $\mathbf{u} \in \mathcal{U} \subseteq \mathbb{R}^m$ is the control input,
We use unified notation where context determines the temporal domain.
A control policy $\pi: \mathcal{X} \to \mathcal{U}$ maps states to actions:
\begin{equation}
    \mathbf{u}(t)  = \pi(\mathbf{x}(t)) \label{eq: policy} \ \text{or} \ \textbf{u}_k = \pi(\mathbf{x}_k),
\end{equation} 
with associated cumulative costs:
\begin{equation}
    \mathcal{J}^\pi(\mathbf{x}(t)) = \int_t^\infty e^{-\frac{s-t}{\tau}} l(\mathbf{x}(s), \mathbf{u}(s)) ds \,\, \text{and} \,\,\mathcal{J}^{\pi}(\mathbf{x_k}) = \sum_{i=k}^\infty \gamma^{i-k} l(\mathbf{x}_i, \mathbf{u}_i) \label{eq: continuous cumulative cost}
\end{equation}
Here $l: \mathcal{X} \times \mathcal{U} \rightarrow \mathbb{R}_+$ is the running cost, $e^{-\frac{s-t}{\tau}}$ and $\gamma^{i-k}$ are discount factors for future cost. States $\mathbf{x}(\cdot)$ and controls $\mathbf{u}(\cdot)$ evolve under dynamics~\eqref{eq: continuous-dyn} and policy~\eqref{eq: policy}.
The optimal control seeks a policy $\pi^*$ minimizing $\mathcal{J}^\pi$ for all initial states. The time-invariant formulation ensures stationarity of $\pi^*$ in infinite-horizon settings.

\subsection{Bellman Equation}

Control solutions that optimize~\eqref{eq: continuous cumulative cost} are established through defining the value function. 
The value function represents the minimal cumulative cost:
\begin{equation}
    \mathcal{V}(\mathbf{x}(t)) = \mathcal{J}^{\pi^*}(\mathbf{x}(t)) = \underset{\mathbf{u}[t, \infty)}{\text{min}}\ \Bigg[\int_t^\infty e^{-\frac{s-t}{\tau}} l(\mathbf{x}(s), \mathbf{u}(s)) ds\Bigg], \label{eq: continuous value function}
\end{equation}
or in discrete-time
\begin{equation}
    \mathcal{V}(\mathbf{x}_k) = \mathcal{J}^{\pi^*}(\mathbf{x}_k) 
    = \underset{\mathbf{u}_k, \mathbf{u}_{k+1},\dots}{\text{min}}\Bigg[\sum_{i=k}^\infty \gamma^{i-k} l(\mathbf{x}_i, \mathbf{u}_i)\Bigg],
\end{equation}
where $\mathbf{u}[t,\infty)$ denotes continuous control trajectories and $\{\mathbf{u}_k, \dots\}$ discrete action sequences.
According to the principle of optimality~\citep{bellman1957dynamic}, the value function must satisfying a
necessary condition known as Bellman equation:
\begin{equation}
    \mathcal{V}(\mathbf{x}_k) =  \underset{\mathbf{u}_k}{\min} \Big[l(\mathbf{x}_k, \mathbf{u}_k) + \gamma \mathcal{V}(\mathbf{x}_{k+1})\Big], \label{eq: Bellman equation}
\end{equation}
or Hamilton-Jacobi-Bellman(HJB) equation for the continuous-time case:
\begin{equation}
    \frac{1}{\tau} \mathcal{V}(\mathbf{x}(t)) = \underset{\mathbf{u}(t) \in \mathcal{U}}{\min} \ \Bigg[ l(\mathbf{x}(t), \mathbf{u}(t)) + \frac{\partial \mathcal{V}(\mathbf{x}(t))}{\partial \mathbf{x}}^\top f(\mathbf{x}(t), \mathbf{u}(t)) \Bigg]. \label{eq: HJB}
\end{equation}
The optimal policy is given by solving right-hand side of the equation, that is 
\begin{equation}
    \mathbf{u}_k = \pi^*(\mathbf{x}_k) = \underset{\mathbf{u} \in \mathcal{U}}{\text{argmin}} \ \Big[ l(\mathbf{x}_k, \mathbf{u}) + \gamma\mathcal{V}(\mathbf{x}_{k+1}) \Big] \label{eq: time-discrete optimal policy given value function}
\end{equation}
for discrete-time system, and 
\begin{equation}
    \mathbf{u}^*(t) = \pi^*(\mathbf{x}(t)) =  \underset{\mathbf{u} \in \mathcal{U}}{\text{argmin}} \ \Bigg[ l(\mathbf{x}(t), \mathbf{u}) + \frac{\partial \mathcal{V}(\mathbf{x}(t))}{\partial \mathbf{x}}^\top f(\mathbf{x}(t), \mathbf{u}) \Bigg] \label{eq: optimal policy given value function}
\end{equation}
for continous-time system.

\subsection{Solving Bellman Equation}

At the heart of value-based approaches~\citep{doya2000reinforcement,sutton2018reinforcement} for solving Bellman equation lies in minimizing the residuals for Bellman Equation, known as temporal difference(TD) error:
\begin{equation}
    \delta(\mathbf{x}_k) \triangleq l(\mathbf{x}_k, \mathbf{u}_k^*) + \gamma\hat{\mathcal{V}}_\theta(f(\mathbf{x}_k,\mathbf{u}_k^*)) - \hat{\mathcal{V}}_\theta(\mathbf{x}_k) \label{eq: TD error},
\end{equation}
or
\begin{equation}
    \delta (\mathbf{x}(t)) \triangleq  l(\mathbf{x}(t), \mathbf{u^*}(t)) + \frac{\partial \mathcal{\hat{V}}_\theta(\mathbf{x}(t))}{\partial \mathbf{x}}^\top f(\mathbf{x}(t), \mathbf{u^*}(t)) - \frac{1}{\tau} \mathcal{\hat{V}_\theta} (\mathbf{x}(t) \label{eq: continuous-time TD},
\end{equation}
where, $\mathcal{\hat{V}_\theta}: \mathcal{X} \times \Theta \rightarrow \mathbb{R}$ is the candidate solution to Bellman equation. 
In the tabular setting, minimization is applied to the entire state space via value iteration (see Appendix~\ref{sec: tabular}). In contrast, in continuous state spaces, we can only update a subset of states. These subsets are collected either through Monte Carlo rollouts~\citep{mnih2015human, lillicrap2015continuous} or via simple grid sampling~\citep{shilova2024learning}.
The notation $\mathbf{u}^*$ represents the current estimate of the optimal action that minimizes the right-hand side of the Bellman equation. 
In practice, $\mathbf{u}^*$ is either performed by training another policy network~\citep{lillicrap2015continuous} or solved analytically~\citep{gu2016continuous,lutter2021value}.

%% file: sections/general_sol.tex
\section{On the General Solutions of Bellman Equation}\label{sec: general solutions}

In this section, we study the general solutions of the Bellman equation and their impact on learning and control.
We begin our analysis with a continuous-time linear dynamical problem and then extend the study to the discrete-time case and nonlinear problems.

\subsection{Solutions to Linear Continuous-Time Control Problem}\label{sec: LQR}
Here we study the general solutions to continuous-time control problem (specifically the Linear-Quadratic-Regulator (LQR) problem): 
\begin{equation}
\begin{array}{cc}
      \min \quad & \int_t^\infty \mathbf{x}(s)^\top Q \mathbf{x}(s) + \mathbf{u}(s)^\top R \mathbf{u}(s) ds  \\
     \text{subject to} \quad & \dot{\mathbf{x}}(t) = A \mathbf{x}(t) + B \mathbf{u}(t)
\end{array} \label{eq: continuous-LQR},
\end{equation}
where $A \in \mathbb{R}^{n \times n}$, $B \in \mathbb{R}^{n \times m}$, $Q = Q^\top \succeq 0 \in \mathbb{R}^{n \times n}$, and $R = R^\top \succ 0 \in \mathbb{R}^{m \times m}$ are constant matrices. 
Assuming a quadratic value function
\begin{equation}
    \mathcal{V}(\mathbf{x}) = \mathbf{x}(t)^\top P \mathbf{x}(t), \label{eq: quadratic value function}
\end{equation}
and substituting the value function~\eqref{eq: quadratic value function} and LQR problem setting~\eqref{eq: continuous-LQR} into the HJB equation~\eqref{eq: HJB}, we have 
\begin{equation}
    0 =\underset{\mathbf{u}(t)}{\text{min}} \ \big[\mathbf{u}(t)^\top R \mathbf{u}(t) + 2\mathbf{x}(t)^\top P B\mathbf{u}(t) + \mathbf{x}(t)^\top Q \mathbf{x}(t) + \mathbf{x}(t)^\top P A \mathbf{x}(t) + \mathbf{x}(t)^\top A P \mathbf{x}(t) \big] \label{eq: LQR HJB}. 
\end{equation}
Notably, the minimizing problem on the right-hand side is quadratic in $\mathbf{u}(t)$. Thereforce, we can obtain the analytical optimal control
\begin{equation}
    \mathbf{u}^*(t) = -R^{-1} B^\top P \mathbf{x}(t) \label{eq: opt u for continuous-LQR},
\end{equation}
which yields
\begin{equation}
    \mathbf{x}^\top (t) (A^\top P + P A - PBR^{-1}B^\top P + Q) \mathbf{x} (t)= 0. \label{eq: LQR HJB with optimal u}
\end{equation}
The HJB equation must be satisfied globally; therefore, the following continuous algebraic Riccati equation (ARE) must hold for the matrix $P$:
\begin{equation}
    A^\top P + P A - PBR^{-1}B^\top P + Q = 0 \label{eq: continuous ARE}.
\end{equation}
Conversely, any matrix $P$ that satisfies the ARE~\eqref{eq: continuous ARE} is a solution to HJB equation of the LQR problem. 

\paragraph{Effect of Discount Factor}
In the formulation of LQR problem~\eqref{eq: continuous ARE}, we omit the discount factor $e^{-\frac{s}{\tau}}$. This omission is made only for simplicity and by no means affect the generality of the results we will show shortly. 
Particularly considering the LQR problem with discounted running cost
\begin{equation}
\begin{array}{cl}
      \min \quad & \int_t^\infty e^{-\frac{s}{\tau}}[\mathbf{x}(s)^\top Q \mathbf{x}(s) + \mathbf{u}(s)^\top R \mathbf{u}(s)] ds  \\
     \text{subject to} \quad & \dot{\mathbf{x}}(s) = A \mathbf{x}(s) + B \mathbf{u}(s), \quad \mathbf{x}(t) \ \text{given}
\end{array} \label{eq: discounted LQR}
\end{equation}
and its undiscounted counterpart with modified dynamics
\begin{equation}
\begin{array}{cl}
      \min \quad & \int_t^\infty \mathbf{x}(s)^\top Q \mathbf{x}(s) + \mathbf{u}(s)^\top R \mathbf{u}(s) ds  \\
     \text{subject to} \quad & \dot{\mathbf{x}}(s) = (A-\frac{1}{2\tau} I) \mathbf{x}(s) + B \mathbf{u}(s), \quad \mathbf{x}(t) \ \text{given}
\end{array} \label{eq: modified LQR},
\end{equation}
we have the following result.
\begin{theorem}\label{thm: discounted don't matter}
The value function $\mathcal{V}(\mathbf{x}) = \mathbf{x}^\top P \mathbf{x}$ satisfies the HJB equation for the discounted LQR problem~\eqref{eq: discounted LQR} if and only if it satisfies the HJB equation for the undiscounted LQR problem with modified dynamics~\eqref{eq: modified LQR}.
\end{theorem}
\begin{proof}
Substituting the discounted LQR problem~\eqref{eq: discounted LQR}, $\mathcal{V}(\mathbf{x}) = \mathbf{x}^\top P \mathbf{x}$, and policy $\mathbf{u}^*(t) = -R^{-1}B^\top P \mathbf{x}(t)$ into HJB equation~\eqref{eq: HJB} yields 
\begin{equation}
    \mathbf{x}^\top(t) (A^\top P + PA - PBR^{-1}B^\top P +Q) \mathbf{x}^\top(t) = \frac{1}{\tau}\mathbf{x}^\top(t) P \mathbf{x}(t) \label{eq: HJB undiscounted LQR}.
\end{equation}
Rearranging the equation~\eqref{eq: HJB undiscounted LQR}, we have
\begin{equation}
    \mathbf{x}^\top(t) [(A - \frac{1}{2\tau} I)^\top P + P(A -\frac{1}{2\tau} I) - PBR^{-1}B^\top P +Q] \mathbf{x}^\top(t) = 0, \label{eq: HJB undiscounted LQR rearrange}
\end{equation}
which matches equation~\eqref{eq: LQR HJB with optimal u} with modified dynamics $\dot{\mathbf{x}}(s) = (A-\frac{1}{2\tau} I) \mathbf{x}(s) + B \mathbf{u}(s)$.
\end{proof}
The stability of opened-loop system $\dot{\mathbf{x}}(s) = A\mathbf{x}(s)$ dependent solely on the eigenvalues of $A$. 
Specially, the system is more stable as the eigenvalues of $A$ become more negative.  
Therefore, an interesting interpretation of Theorem~\ref{thm: discounted don't matter} is that, in terms of the value function, incorporating a discount factor into the running cost is equivalent to controlling a more stable system. 

\paragraph{Solutions to Algebraic Riccati Equation}
Here, we introduce a subspace invariant method~\citep{lancaster1995algebraic} to obtain a general solution for the ARE~\eqref{eq: continuous ARE}.
We begin by construct a $2n\times2n$ a Hamiltonian matrix:
\begin{equation}
    H = \begin{bmatrix}
    A &-BR^{-1}B^\top\\
    -Q &-A^\top
    \end{bmatrix} \label{eq: hamiltonian}.
\end{equation}

\begin{definition}
We say $\mathcal{M}$ is an invariant subspace of $H$ if $H \mathbf{x} \in \mathcal{M}$ for all $\mathbf{x} \in \mathcal{M}$.
\end{definition}
\begin{theorem} \label{thm: solution to ARE}
Suppose $[P_1^\top,P_2^\top]^\top$forms a basis of invariant subspace of $H$~\eqref{eq: hamiltonian}, where $P_1, P_2$ are $n \times n$ matrices and $P_1$ invertible. Then 
\begin{equation}
    P=P_2P_1^{-1} \label{eq: sol P}
\end{equation}
is a solution to ARE~\eqref{eq: continuous ARE}.
\end{theorem}
\begin{proof}
By definition of invariant subspace, there exist a $n \times n$ matrix $T$ such that
\begin{equation}
    H \begin{bmatrix}P_1\\P_2\end{bmatrix} = \begin{bmatrix}P_1\\P_2\end{bmatrix} T.
\end{equation}

By $P_1$ invertible, we have
\begin{equation}
    \begin{bmatrix}
    A &-BR^{-1}B^\top\\
    -Q &-A^\top
    \end{bmatrix} \begin{bmatrix}P_1\\P_2\end{bmatrix} P_1^{-1} = \begin{bmatrix}P_1\\P_2\end{bmatrix} P_1^{-1}P_1TP_1^{-1} =  \begin{bmatrix}P_1\\P_2\end{bmatrix} P_1^{-1} \hat{T}.
\end{equation}
so that 
\begin{equation}
    \begin{bmatrix}
    A &-BR^{-1}B^\top\\
    -Q &-A^\top
    \end{bmatrix} \begin{bmatrix}
        I \\ P_2 P_1^{-1}
    \end{bmatrix} = \begin{bmatrix}
        \hat{T} \\ P_2 P_1^{-1} \hat{T}
    \end{bmatrix}
\end{equation}
These give two set of matrix equation
\begin{align}
    \hat{T} = A - BR^{-1}B^\top P_2 P_1^{-1},\\
    -Q - A^\top P_2P_1^{-1} = P_2 P_1^{-1} \hat{T}.
\end{align}
Substituting the first equation into second we obtain
\begin{equation}
Q + A^\top P_2 P_1^{-1} + P_2 P_1^{-1} A - P_2P_1^{-1} BR^{-1} B^\top P_2 P_1^{-1} = 0,
\end{equation}
which follow the ARE form~\eqref{eq: continuous ARE} and complete the proof.
\end{proof}
We can then apply the eigen-decomposition\footnote{If $H$ is not diagonalizable, we can still apply the Schur decomposition. The results remain the same in this case.} to obtain invariant subspace for $H$:
\begin{equation}
    H = V \Lambda V^{-1} = \begin{bmatrix}
        P_1 &S_1\\ P_2 &S_2
    \end{bmatrix} \begin{bmatrix}
        \Lambda_1 &0\\
        0 &\Lambda_2 
    \end{bmatrix} \begin{bmatrix}
        P_1 &S_1\\ P_2 &S_2
    \end{bmatrix}^{-1}, \label{eq: H decompose}
\end{equation}
where $[S_1^\top, S_2^\top]^\top$ is $2n\times n$ matrices that store other $n$ eigenvectors.
By choosing different combination of columns in the eigen-space, we can form $\binom{2n}{n}$ different invariant subspace and each of them yields at least one solution to HJB equation.

\subsection{Behavior of Closed-Loop System under General Solutions}
Substituting the policy~\eqref{eq: opt u for continuous-LQR} into the open-loop dynamics in~\eqref{eq: continuous-LQR} yields the dynamics of the closed-loop system:
\begin{equation}
    \dot{\mathbf{x}} (t) = (A - B R^{-1} B^\top P) \ \mathbf{x}(t) = A_{cl} \mathbf{x}(t) \label{eq: closed loop dyn}.
\end{equation}
Here, we study the behavior of closed-loop system~\eqref{eq: closed loop dyn} by relating the eigenvalues of $A_{cl}$ to $H$.

\begin{lemma}\label{lemma: no imaginary eigenvalues}
If the pair $(A,B)$ is controllable and the pair $(Q,A)$ is observable, the Hamiltonian matrix $H$~\eqref{eq: hamiltonian} has no pure imaginary or zero eigenvalues. 
\end{lemma}
\begin{proof}
See Appendix~\ref{sec: additional proof}.
\end{proof}
\begin{theorem}\label{theorem: number of eigenvalues}
If $\lambda$ is an eigenvalue of $H$, then $-\lambda$ is also an eigenvalues. Furthermore, if the pair $(A,B)$ is controllable and the pair $(Q,A)$ is observable, $H$ has n positive eigenvalues and n negative eigenvalues.
\end{theorem}
\begin{proof}
Let $J = \begin{bmatrix}0 &I\\-I &0\end{bmatrix}$, Then we have $HJ=(HJ)^\top$. 

So $H \mathbf{x} = \lambda \mathbf{x}$ implies 
\begin{align}
J^\top H J J^\top \mathbf{x} = \lambda J^\top \mathbf{x} \nonumber\\
H^\top (J^\top \mathbf{x}) = -\lambda (J^\top \mathbf{x}) \nonumber
\end{align}
Hence, $-\lambda$ is an eigenvalue of $H^\top$ and must be an eigenvalue of $H$. 

By Lemma~\ref{lemma: no imaginary eigenvalues}, the eigenvalues only contains real number, so it must have n positive eigenvalues and $n$ negative eigenvalues.
\end{proof}
\begin{theorem} \label{thm: close-loop eigenvalues}
Suppose $[P_1^\top,P_2^\top]^\top$ are invariant subspace of $H$~\eqref{eq: hamiltonian} with eigenvalues $\Lambda_1 = \text{diag}(\lambda_1, \dots, \lambda_n)$ and $P=P_2P_1^{-1}$, then the characteristic matrix of closed-loop system $A_{cl} = (A - B R^{-1} B^\top P)$ has eigenvalues $\Lambda_1$. 
\end{theorem}
\begin{proof}
By invariant subspace, we have
\begin{equation}
     \begin{bmatrix}
    A &-BR^{-1}B^\top\\
    -Q &-A^\top
    \end{bmatrix} \begin{bmatrix}P_1\\P_2\end{bmatrix} = \begin{bmatrix}P_1\\P_2\end{bmatrix} \Lambda_1,
\end{equation}
which give
\begin{equation}
\begin{array}{cc}
     & AP_1 - BR^{-1} B^\top P_2 = P_1 \Lambda_1 \\
     & -QP_1 -A^\top P_2 = P_2 \Lambda_1
\end{array}
\end{equation}
Substituting $P_2 = P P_1$ into the first equation, we obtain
\begin{equation}
    (A - BR^{-1} B^\top P) P_1 = P_1 \Lambda_1, 
\end{equation}
which completes the proof.
\end{proof}
\begin{corollary}\label{corollary: only one sol make stable}
Given the conditions in Lemma~\ref{lemma: no imaginary eigenvalues}, and assume $H$ diagonalizable, there exists at least $\binom{2n}{n}$ matrix $P$s that satisfies ARE~\eqref{eq: continuous ARE}, but \textbf{only one} make close-loop system~\eqref{eq: closed loop dyn} stable.
\end{corollary}
\begin{proof}
By picking different sets of $n$ eigenvector from $2n$-dimensional eigenspace of~\eqref{eq: H decompose}, we can construct $\binom{2n}{n}$ invariant subspaces $\mathcal{M}$s, each invariant subspace corresponding to a solution $P$.
By Theorem~\ref{theorem: number of eigenvalues} and Theorem~\ref{thm: close-loop eigenvalues}, the closed-loop system is stable, only when eigenvectors with all negative $\lambda$ are selected.  
\end{proof}
A similar result to Corollary~\ref{corollary: only one sol make stable} holds for the discounted cost setting, where at most one solution leads to a stable closed-loop system. 
More details are provided in Appendix~\ref{sec: additional proof}.
\subsection{The Discrete-Time Case}
The results for the discrete-time case are similar to their continuous-time counterpart; however, the proof is more involved. 
Therefore, we present the main theorems for the discrete-time system in the main text and leave the proof in Appendix~\ref{sec: additional proof}.

We consider the discrete-time LQR prblem:
\begin{equation}
\begin{array}{cl}
      \min \quad & \sum_{i=t}^\infty \mathbf{x}_i^\top Q \mathbf{x}_i + \mathbf{u}_i^\top R \mathbf{u}_i  \\
     \text{subject to} \quad & \mathbf{x}_{i+1} = A \mathbf{x}_i + B \mathbf{u}_i, \quad i=t, t+1, \dots 
\end{array}. \label{eq: discrete LQR}
\end{equation}
The discrete-time algebraic-Riccati equation of~\eqref{eq: discrete LQR} is given by
\begin{equation}
    P = Q + A^\top P A - A^\top P B (R + B^\top P B)^{-1} B^\top P A. \label{eq: discrete ARE} 
\end{equation}
Similar to the continuous-time case, we construct the Hamiltonian matrix as:
\begin{equation}
    H = \begin{bmatrix}
        A + BR^{-1}B^\top A^{-\top} Q & - BR^{-1}B^\top A^{-\top}\\
        -A^{{-\top}}Q &A^{-\top}
    \end{bmatrix} \label{eq: discrete-Hamitonian}.
\end{equation}

\begin{theorem}\label{thm: solutions to discrete-ARE}
Suppose $[P_1^\top,P_2^\top]^\top$ are invariant subspace of discrete-time Hamitonian matrix $H$~\eqref{eq: discrete-Hamitonian}, then $P=P_2P_1^{-1}$ is the solution to discrete-time algebraic-Riccati equation~\eqref{eq: discrete ARE}.
\end{theorem}
\begin{proof}
See Appendix~\ref{sec: additional proof}.
\end{proof}

In the discrete-time problem, the dynamics of closed-loop is given by
\begin{equation}
    \mathbf{x}_{i+1} =(A-B(R+B^\top P B)^{-1}B^\top PA)\mathbf{x}_i = A_{cl} x_i, \quad i=t, t+1, \dots,  
\end{equation}
where the stability is determined by whether all eigenvalues of $A_{cl}$  are within unit circle. 
We now relate the eigenvalues of closed-loop system matrix $A_{cl}$ with Hamiltonian $H$ for the discrete-time system.

\begin{theorem}\label{thm: number of eigenvalues compare to unit cycle}
If $\lambda$ is an eigenvalue of the discrete-time Hamiltonian matrix $H$~\eqref{eq: discrete-Hamitonian}, then $\frac{1}{\lambda}$ is also a eigenvalue of $H$. Moreover, if the pair $(A,B)$ is controllable and the pair $(Q,A)$ is observable, then $H$ has n eigenvalues within the unit circle and n eigenvalues outside the unit cycle.
\end{theorem}
\begin{proof}
    See Appendix~\ref{sec: additional proof}.
\end{proof}

\begin{theorem}\label{thm: eigenvalues and closed-loop system}
Suppose $[P_1^\top,P_2^\top]^\top$ are invariant subspace of $H$~\eqref{eq: discrete-Hamitonian} with eigenvalues $\Lambda_1 = \text{diag}(\lambda_1, \dots, \lambda_2)$ and $P=P_2P_1^{-1}$, then characteristic matrix of closed-loop system $A_{cl} = (A-B(R+B^\top P B)^{-1}B^\top PA)$ has eigenvalues $\Lambda_1$. 
\end{theorem}
\begin{proof}
See Appendix~\ref{sec: additional proof}
\end{proof}
Given the Theorem~\ref{thm: number of eigenvalues compare to unit cycle} and Theorem~\ref{thm: eigenvalues and closed-loop system}, we have similar results for the discrete-time system, i.e., only one among $\binom{2n}{n}$ general solution to Bellman equation leads to stable closed-loop system. 
\subsection{The Nonlinear Case}
Here, we discuss the general solutions to HJB equation for nonlinear dynamics.
In many applications, optimal control problems are highly structured.
For example, in robotics, a common assumption is that the dynamics are control-affine and the running cost is separable and convex in the control input.
In such scenarios, the optimal control problem can often be solved analytically~\citep{lutter2021value}.
Specifically, consider the optimal control problem formulated as follows:
\begin{equation}
\begin{array}{cl}
      \min \quad & \int_t^\infty e^{-\frac{s}{\tau}}[l_1(\mathbf{x}(s)) + \mathbf{u}(s)^\top R \mathbf{u}(s)] ds  \\
     \text{subject to} \quad & \dot{\mathbf{x}}(s) = f_1( \mathbf{x}(s)) + f_2(\mathbf{x}(s)) \mathbf{u}(s), \quad \mathbf{x}(t) \ \text{given},
\end{array} \label{eq: nonlinear control problem}
\end{equation}
where $f_1: \mathcal{X} \rightarrow \mathbb{R}^n, f_2: \mathcal{X} \rightarrow \mathbb{R}^{n\times m}, l_1: \mathcal{X} \rightarrow \mathbb{R}_+$ and $R =R^\top \succ 0 \in \mathbb{R}^{m\times m}$.

The optimal policy synthesis problem~\eqref{eq: optimal policy given value function} is quadratic in control and can be obtained analytically:
\begin{align}
    \mathbf{u}^*(t) =  &\underset{\mathbf{u} \in \mathcal{U}}{\text{argmin}} \ \Bigg[ l_1(\mathbf{x}(t)) + \mathbf{u}^\top R \mathbf{u} + \frac{\partial \mathcal{V}(\mathbf{x}(t))}{\partial \mathbf{x}}^\top  (f_1( \mathbf{x}(t)) + f_2(\mathbf{x}(t)) \mathbf{u}) \Bigg] \nonumber\\
    = &-\frac{1}{2} R^{-1} f_2^\top(\mathbf{x}(t)) \frac{\partial \mathcal{V}(\mathbf{x}(t))}{\partial \mathbf{x}} \label{eq: nonlinear optimal u}.
\end{align}
Substituting to~\eqref{eq: nonlinear optimal u} into the HJB equation~\eqref{eq: HJB} yields the elliptic partial differential equation:
\begin{equation}
    \frac{1}{\tau} \mathcal{V}(\mathbf{x}(t)) = -\frac{1}{4} \frac{\partial \mathcal{V}(\mathbf{x}(t))}{\partial \mathbf{x}}^\top f_2(\mathbf{x}(t)) R^{-1} f_2(\mathbf{x}(t))^\top \frac{\partial \mathcal{V}(\mathbf{x}(t))}{\partial \mathbf{x}} + f_1(\mathbf{x}(t))^\top \frac{\partial \mathcal{V}(\mathbf{x}(t))}{\partial \mathbf{x}} + l_1(\mathbf{x}(t)) . \label{eq: pde}
\end{equation}
Thus, solving the HJB equation for the nonlinear control problem~\eqref{eq: nonlinear control problem} is equivalent to solving the elliptic partial differential equation~\eqref{eq: pde}.

In general, the elliptic partial differential equation~\eqref{eq: pde} does not admit a unique solution. However, the discussion of general solutions to nonlinear elliptic partial differential equations is beyond the scope of this paper, and we refer readers to~\citet{han2011elliptic,evans2022partial} for further details.

%% file: sections/canonical_failure.tex
\section{A Canonical Failure Mode in Learning Value Functions}\label{sec: canonical failure}
\begin{wrapfigure}{h}{0.5\textwidth}
    \vspace{-20pt}
    \begin{center}
    \includegraphics[width=0.45\textwidth]{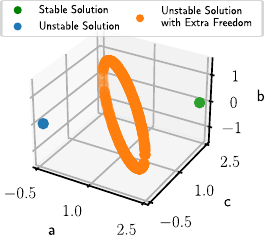}
    \end{center}
    \caption{Solution to LQR: $\mathcal{V}(\mathbf{x}) = \mathbf{x}^\top P \mathbf{x}$, where $P$ is given by~\eqref{eq: toy sol}. 
    The green dot indicates the stable solution, the blue dot marks the unstable solution, and the orange ring represents additional solutions arising from the noninvertibility of $P_1$.} 
    \vspace{-70pt}
    \label{fig:ARE Sol}
\end{wrapfigure}
A generic neural network may converge to any solution of the Bellman equation by minimizing the TD error.
Moreover, it is more likely for the network to converge to undesired solution given the imbalance ratio of solutions.
Here, we present an example to show such failure mode do exist in practice. 

Consider the time-continuous LQR~\eqref{eq: continuous-LQR} with the parameters $A=B=Q=R=\begin{bmatrix}1 &0\\0 &1\end{bmatrix}$.

We obtain the analytical solution by using the subspace invariant method discussed previously:
\begin{align}
    H &= \begin{bmatrix}
        1 &0 &-1 &0\\
        0 &1 &0 &-1\\
        -1 &0 &-1 &0\\
        0 &-1 &0 &-1
    \end{bmatrix} = V \Lambda V^{-1},
\end{align}
where 
\begin{equation}
    V = \begin{bmatrix}
    1+\sqrt{2} &0 &1-\sqrt{2} &0 \\
    0 &1+\sqrt{2} &0 &1-\sqrt{2} \\
    -1 &0 &-1 &0 \\
    0 &-1 &0 &-1 
    \end{bmatrix}, \Lambda =\begin{bmatrix}
        \sqrt{2} &0 &0 &0\\
        0 &\sqrt{2} &0 &0\\
        0 &0 &-\sqrt{2} &0\\
        0 &0 &0 &-\sqrt{2}
    \end{bmatrix} .
\end{equation}

By selecting columns from the matrix $V$, we can obtain different solution matrices $P$. 
Note that when the first and third or second and fourth columns are selected, i.e., $P_1 = \begin{bmatrix}1+\sqrt{2} &1-\sqrt{2}\\0 &0\end{bmatrix}, P_2=\begin{bmatrix}-1 &-1\\0 &0\end{bmatrix}$ or $P_1 = \begin{bmatrix}0 &0\\1-\sqrt{2} &1-\sqrt{2}\end{bmatrix}, P_2=\begin{bmatrix}0 &0\\-1 &-1\end{bmatrix}$, the matrix $P_1$ is not invertible, introducing additional degrees of freedom. 
This leads to an infinite number of unstable solutions.In summary, the analytical solution can be expressed as 
$P=\begin{bmatrix}a &b\\b &c\end{bmatrix}$, where
\begin{align}
\begin{cases}
a = 1 + \sqrt{2}, b=0, c=1 + \sqrt{2} \ \text{(Stable Solution)}\\
a = 1 - \sqrt{2}, b=0, c=1 - \sqrt{2} \ \text{(Unstable Solution)}\\
a = z, b=\sqrt{2z+1-z^2}, c=2-z, z\in(1-\sqrt{2}, 1+\sqrt{2}) \ \text{(Unstable with Extra Freedom)}
\end{cases}. \label{eq: toy sol}
\end{align}
These solutions are visualized in Figure~\ref{fig:ARE Sol}.

Now, we solve the LQR problem with a learning approach. 
Specifically, we use a 3-layer MLP architecture with 128 neurons in each layer and the ELU activation function to approximate the value function. 
The weights are initialized with Lecun normal~\citep{klambauer2017self} and the biases are initialized to zero.
We collect $10^4$ samples uniformally from $[-2, 2] \times[-2, 2]$.
The neural network is then trained by minimizing the mean square error of TD residuals~\eqref{eq: continuous-time TD} using Adam optimizer with learning rate $10^{-3}$.
We train the networks for $1000$ epochs and achieve a mean square error below $10^{-4}$.

\begin{figure}[hb]
    \begin{center}
    \vskip -5pt
    \includegraphics[width=0.95\textwidth]{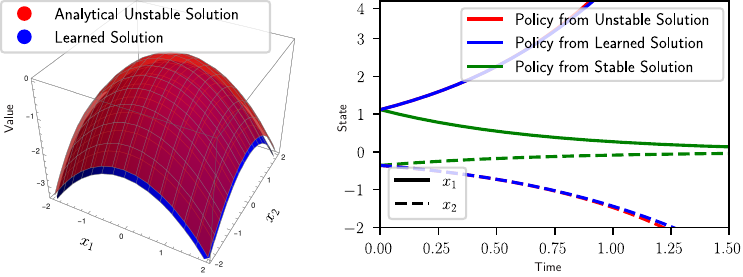}
    \end{center}
    \vskip -12pt
    \caption{Learned solution to LQR. Left figure: the learned value function converges to the unstable solution, up to an additive constant. Right figure: both the learned and analytical unstable solutions yield identical diverging trajectories, whereas the stable solution converges to the origin.}
    \vskip -8pt
    \label{fig:lqr sol}
\end{figure}
Figure~\ref{fig:lqr sol} compares the learned and analytical solutions. In this case, the neural network converges to the unstable solution, and the associated policy causes trajectories to diverge. 
The issue of non-unique solutions also arises in \textbf{nonlinear} dynamics, as we will demonstrate in Section~\ref{sec:potential solutions}. 
Moreover, we find the solution to which the network converges is highly sensitive to weight initialization (see Appendix~\ref{sec: more details LQR}), suggesting that multiple solutions may partly explain the sensitivity of value-based approaches to hyperparameters and random seeds~\citep{ceron2024consistency}.

%% file: sections/potential_solutions.tex
\section{What Can We Do}\label{sec:potential solutions}

In this section, we discuss potential solutions to address the failure of the value-based approach in finding the optimal control policy due to the existence of general solutions to the Bellman equation.

\subsection{Approach 1: Adding Boundary Conditions}

Setting boundary conditions is one of the most common techniques for isolating the stable solution of the Bellman equation~\citep{mnih2015human, lillicrap2015continuous}. 
However, selecting appropriate boundary conditions can be challenging in practice. 
Insufficient boundary conditions may fail to guarantee uniqueness, while inconsistent ones can result in no solution at all.
\paragraph{Insufficient Boundary Conditions}
\begin{wrapfigure}{r}{0.4\textwidth}
    \begin{center}
    \vskip -20pt
    \includegraphics[width=0.4\textwidth]{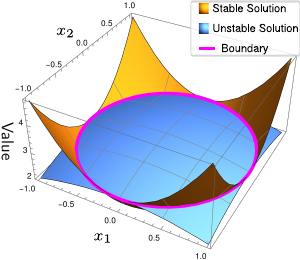}
    \end{center}
    \vskip -5pt
    \caption{Example of insufficient boundary conditions: both solutions share the same boundary value, yet one yields a stable closed-loop system while the other results in instability.} 
    \label{fig:insufficient boundary conditions}
    \vskip -60pt
\end{wrapfigure}
Considering the LQR example in Section~\ref{sec: canonical failure}, we can define a unit circle boundary $\mathcal{B} = \{\mathbf{x}: \mathbf{x}^\top \mathbf{x} =1\}$ and set the Dirichlet boundary condition: $\mathcal{V}(\mathbf{x}) = 1+\sqrt{2}$, for $\mathbf{x} \in \mathcal{B}$.

Under the specified boundary conditions, the stable solution, the stable solution $\mathcal{V}(\mathbf{x}) = \mathbf{x}^\top \begin{bmatrix}1+\sqrt{2} &0\\0 &1+\sqrt{2} \end{bmatrix} \mathbf{x}$ still solves the associated Bellman equation.
However, there is another solution $\mathcal{V}(\mathbf{x}) = \mathbf{x}^\top \begin{bmatrix}1-\sqrt{2} &0\\0 &1-\sqrt{2} \end{bmatrix} \mathbf{x} + 2\sqrt{2}$, which meets both the Bellman equation and the boundary conditions. 
This solution is obtained by adding an offset of $ 2\sqrt{2}$ to the unstable solution $\mathcal{V}(\mathbf{x}) = \mathbf{x}^\top \begin{bmatrix}1-\sqrt{2} &0\\0 &1-\sqrt{2} \end{bmatrix} \mathbf{x}$ by $ 2\sqrt{2}$.
This phenomenon occurs because the Bellman equation and the resulting policy depend only on the relative values of the value function rather than its absolute magnitude.

\paragraph{Inconsistent Boundary Conditions}
Here, we retain the Dirichlet boundary conditions from the previous paragraph i.e., $\mathcal{V}(\mathbf{x}) = 1+\sqrt{2}$, if $\mathbf{x} \in \mathcal{B}_1$, where $\mathcal{B}_1 = \{\mathbf{x}: \mathbf{x}^\top \mathbf{x} =1\}$.
Additionally, we introduce a new boundary: $\mathcal{B}_2 = \{\mathbf{x}: \mathbf{x}^\top \mathbf{x} = 0.5 \}$, and set the Dirichlet boundary condition for this new boundary as: $\mathcal{V}(\mathbf{x}) = 0$ if $\mathbf{x} \in \mathcal{B}_2$. 
In this case, no general solution to Bellman equation can satisfy both boundary conditions.

\subsection{Approach 2: Special Neural Architectures}
For many optimal control problems, the goal is to stabilize the system to an equilibrium point $(\mathbf{x}_\text{eq}, \mathbf{u}_\text{eq})$ or track a reference trajectory.
In this scenario, we can design specialized architectures to exclude the unstable general solutions that the Bellman equation can admit.

\begin{theorem}\label{thm: stable sol with lyapunov conditions}
If $\mathcal{V}(\mathbf{x})$ is a solution to HJB equation~\eqref{eq: HJB} without discounted running cost, where $\mathcal{V}(\mathbf{x}_\text{eq}) = 0$ and $\mathcal{V}(\mathbf{x}) > 0$ for any $\mathbf{x} \not=\mathbf{x}_\text{eq}$, then the closed-loop system under the policy given by~\eqref{eq: optimal policy given value function} is stable. Furthermore, if $l(\mathbf{x}, \mathbf{u}) > 0$ for any state-control pair $(\mathbf{x}, \mathbf{u})$ other than $(\mathbf{x}_\text{eq}, \mathbf{u}_\text{eq})$, then $\lim_{s\rightarrow\infty} \|\mathbf{x}(s) - \mathbf{x}_\text{eq}\| = 0$ for any initial condition $\mathbf{x}(t)$.
\end{theorem}
\begin{proof}
See Appendix~\ref{sec: additional proof}.
\end{proof}
The key insight here is the value function $\mathcal{V}(\mathbf{x})$ can naturally serve as a Lyapunov function for the closed-loop system with the policy given by~\eqref{eq: optimal policy given value function}. 
Particularly, Bellman equation guarantees $\dot{\mathcal{V}}(\mathbf{x}) < 0$ under the policy~\eqref{eq: optimal policy given value function}. 
Additionally, the definition of the optimal control problem and value function ensure the positive-definiteness conditions for the Lyapunov function, i.e., $\mathcal{V}(\mathbf{x}_\text{eq})=0$ and $\mathcal{V}(\mathbf{x})>0$ for any $\mathbf{x} \not= \mathbf{x}_\text{eq}$.

Theorem~\ref{thm: stable sol with lyapunov conditions} inspires us to design a neural architecture that preserves the positive-definiteness, i.e., ensuring, regardless of how weight $\theta$ are set, we always have $\hat{\mathcal{V}}_\theta(\mathbf{x}_\text{eq})=0$ and $\hat{\mathcal{V}}_\theta(\mathbf{x})>0$ for any $\mathbf{x} \not= \mathbf{x}_\text{eq}$.
We provide a positive-definite architecture below
\begin{align}
    \hat{\mathcal{V}}_\theta(\mathbf{x}) &= h_{N_h}^\top h_{N_h} + \epsilon \|\mathbf{x} - \mathbf{x}_{eq}\|^2_2, \label{eq: architecture}
\end{align}
where 
\begin{equation}
    h_{k+1} = \sigma(\theta_k h_{k}) \ \text{for} \ k=1, 2, \dots, N_h \, ,\text{and} \, h_1 = \mathbf{x} - \mathbf{x}_{eq}.
\end{equation}
Here, $\sigma$ represents nonlinear activation function that satisfies $\sigma(0) = 0$ (e.g., ReLU, tanh), and $\epsilon$ is a small constant (e.g., $10^{-3}$).
For any $\mathbf{x} \neq \mathbf{x}_{eq}$, we have $h_{N_h}^\top h_{N_h} \geq 0$ and $\|\mathbf{x} - \mathbf{x}_{eq}\|_2^2 > 0$, which implies that $\hat{\mathcal{V}}_\theta(\mathbf{x}) > 0$.
For $\mathbf{x}_{eq}$, we have $h_1=0$, and by property of $\sigma$, $h_{k+1} = 0$ for $k\geq 1$; consequently, $\hat{\mathcal{V}}_\theta(\mathbf{x}_{eq}) = 0$.
\begin{figure}[ht]
    \vskip -5pt
    \includegraphics[width=0.95\textwidth]{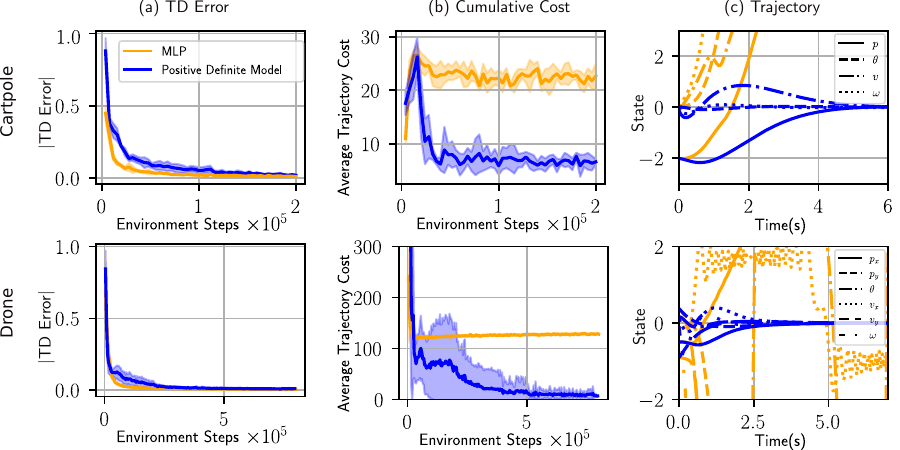}
    \caption{Value learning with MLP and positive-definiteness architecture. 
    For each case, we run the experiments with 5 different seeds, reporting the average performance. 
    The shaded area represents the standard deviation across the runs. The TD error is minimized for both architecture, indicating a solution to Bellman equation is found. However, generic MLP architecture cannot distinguish the stable solution from the others leading to high cumulative cost and diverging behavoir}
    \vskip -5pt
    \label{fig: PD model Learning}
\end{figure}

We apply the proposed architecture to two \textbf{nonlinear} dynamical systems: cartpole and drone.
The drone scenario poses additional challenges due to its intricate dynamics, the need for rapid response, and limited control inputs.
Further details regarding the tasks, learning algorithm, and hyperparameters are provided in Appendix~\ref{sec: experiment setup}.

Figure~\ref{fig: PD model Learning} illustrates the results for both our proposed architecture and a generic MLP.
The absolute TD error is minimized for both architectures, implying that the learning algorithm successfully finds a solution to the Bellman equation in both cases. 
However, the generic MLP architecture cannot distinguish the stable solution from the others, the resulting controller leads to an unstable closed-loop system, causing higher cumulative costs.
These results highlight the presence of non-unique solutions in systems with \textbf{nonlinear} dynamics.
In contrast, by using an architecture that preserves positive definiteness, the learning algorithm consistently converges to the stable solution and achieves lower cumulative costs.

\begin{wraptable}{r}{0.55\textwidth}
    \vskip -5pt
    \caption{Trajectory Cost for Final Policy}
    \begin{center}
        \vskip -10pt
        \begin{tabular}{lllll}
            Methods  &Linear &Cartpole &Drone
            \\ \hline
            Ours &0.87 &6.98 &8.64 \\
            \citet{lutter2020hjb} &4.38 &10.28 &362.31 \\
            \hline
        \end{tabular}
    \end{center}
    \vskip -10pt
    \label{tab: avg trajectory cost}
\end{wraptable}

We compare our method to~\citet{lutter2020hjb,lutter2021value}'s continuous-time RL approach, which employs discount factor i.e., $\frac{s}{\tau}$, scheduling and Lagrangian networks~\citep{lutter2019deep} for learning the stable solution. Evaluations on three tasks (including their original and the more challenging drone system) show our method outperforms theirs, particularly on the drone task (Table~\ref{tab: avg trajectory cost}). 
We find that a large discount factor 
$\frac{s}{\tau}$ complicates optimization, leading the problem essentially ill-defined~\citep{tassa2007least}. Moreover, the Lagrangian network does not guarantee that the equilibrium point is the unique minimizer~\citep{lutter2023robust}—a key property for learning the stable solution. We hypothesize that these factors account for the performance gap between our method and that of~\citet{lutter2020hjb,lutter2021value}.

%% file: sections/related_work.tex
\section{Related Work}\label{sec: related work}
\paragraph{Uniqueness of Solution to Bellman Equation}
The uniqueness of Bellman equations has been studied in both reinforcement learning and control.
The first proof for the uniqueness of the Bellman equation in finite state and action spaces was provided by Richard Bellman~\citep{bellman1957dynamic}. 
A well-known contraction mapping proof was given by~\citet{blackwell1965discounted}, also included in Appendix~\ref{sec: tabular}.
This idea was extended by~\citet{sutton1988learning}. 
The study in continuous state and action spaces was pioneered by~\citet{bertsekas1996stochastic}, who showed the value function is the unique fixed point of the Bellman equation under regularity conditions.

Recent works have explored non-uniqueness in Bellman equations. For example,~\citet{misztela2018nonuniqueness} shows two distinct semi-continuous solutions to an HJB equation, and~\citet{HOSOYA2024102940} finds infinitely many solutions in a special class of HJB equations. Additionally,~\citet{xuan2024uniqueness} shows that a discount factor of one leads to multiple solutions.
In this work, we demonstrate that non-uniqueness is common in practice, and that the solution space can grows exponentially with state dimension.  
We further study the behavior of the closed-loop system and show that only one solution leads to stability, which presents challenges for value-based methods.
Several key theorems in this work are inspired by solutions to algebraic Riccati equations in control~\citep{willems1971least, lancaster1995algebraic}. We provide an alternative, self-contained, linear-algebraic proof for our main results, without relying on functional analysis tools, contributing to the theoretical foundation.

\paragraph{Neural Architectures for Dynamics and Control} 
The need for stable dynamics and control has led to the development of structured architectures.
\citet{kolter2019learning} propose input-convex networks for stable dynamical systems. Meanwhile \citet{lutter2019deep} introduce Lagrangian networks for physically plausible dynamics. These architectures have been applied to control tasks~\citep{lutter2023continuous}. Positive-definiteness architectures are often explored in safety-critical reinforcement learning~\citep{chang2021stabilizing, brunke2022safe, dawson2022safe, dawson2023safe}, particularly due to their use of Lyapunov or barrier functions. In this work, we show a close relationship between value functions and Lyapunov functions, demonstrating how structured architectures can facilitate training and enhance policy performance for value-based methods.

%% file: sections/conclusion.tex
\section{Conclusion}\label{sec: conclusion}
In this work, we show that solutions to the Bellman equation are generally nonunique in continuous state and action spaces, with the solution space potentially growing exponentially with dimension, a phenomenon we term the ``curse of dimensionality in solution space''. 
While one might mitigate this by setting boundary conditions or designing initialization strategies, we demonstrate that this is often challenging in practice. Instead, we propose using structured architectures to constrain the search space to only desirable solutions, leveraging the close relationship between value functions and Lyapunov functions. 

The primary limitation of this architecture is its applicability mainly to control problems that demand system stability.
In problems such as locomotion, the objective is often to achieve a limit cycle~\citep{underactuated} rather than converging to an equilibrium point.
In these scenarios, a more effective initialization scheme may be necessary to find the desired solution, as we discussed in Appendix~\ref{sec: more details LQR}.
Another alternative is to switch to policy-based approaches~\citep{schulman2017proximal, xu2021accelerated}. These methods do not face the challenges of distinguishing solutions that value-based approaches encounter, as they directly optimize the policy with respect to the performance index.

%% file: sections/addtional_proof.tex
\beginSupplementaryMaterials
\appendix
\section{Additional Proof}\label{sec: additional proof}

\paragraph{Proof of Lemma~\ref{lemma: no imaginary eigenvalues}} 
\begin{proof}
We proof Lemma~\ref{lemma: no imaginary eigenvalues} via contradiction. 

Assume $\lambda + \Bar{\lambda} = 0$ are pure imaginary numbers or zeros and 
\begin{equation}
    H \begin{bmatrix}
        x\\
        y
    \end{bmatrix} = \lambda \begin{bmatrix}
        x\\
        y
    \end{bmatrix},
\end{equation}
where $x$ and $y$ are not both zero. We use upper bar to denote the conjugate of the vector.

We have
\begin{align}
    A x -BR^{-1} B^\top y = \lambda x \\
    -Q x - A^\top y = \lambda y,
\end{align}
By replacing the second equation into first and multiplying both side with $\bar{y}$ yields
\begin{equation}
    \bar{y} BR^{-1} B^\top y = -(\bar{x} Q + \bar{\lambda} \bar{y} ) x - \lambda \bar{y} x = -\bar{x} Q x.
\end{equation}
Since $BR^{-1} B^\top \succeq 0$ and $Q \succeq 0$, we conclude $Q x$ and $BR^{-1} B^\top y$ = 0.
Thus 
\begin{equation}
    \begin{bmatrix}
        A -\lambda I \\
        Q
    \end{bmatrix} x =0.
\end{equation} 
If $x \neq 0$, then the above equation contradict the observability of $(Q, A)$ by Popov-Belevitch-Hautus test~\citep{hespanha2018linear}.
Similarly, if we have $y \neq 0$, we have 
\begin{equation}
    \bar{x} [BR^{-1}B^\top, A +\bar{\lambda}I] = 0,
\end{equation}
which contradict the controllability of $(A,B)$.
\end{proof}
\paragraph{Extension of Corollary~\ref{corollary: only one sol make stable} to the Discounted Setting:}
\begin{theorem}\label{thm: discounted LQR sol and closed-loop system}
For the discounted LQR problem~\eqref{eq: discounted LQR}, \textbf{at most} one solution to the HJB equation~\eqref{eq: HJB undiscounted LQR} results in a stable closed-loop system.
\end{theorem}
\begin{proof}
We prove Theorem~\ref{thm: discounted LQR sol and closed-loop system} by relating the eigenvalues of the closed-loop system matrix $A_{cl} = A - BR^{-1}BP$ to its undiscounted counterpart~\eqref{eq: modified LQR}.

Suppose $P$ is a solution to the discounted LQR problem~\eqref{eq: discounted LQR}. By Theorem~\ref{thm: discounted don't matter}, it is also a solution to its undiscounted counterpart~\eqref{eq: modified LQR}.
Then, we have a following relationship between the closed-loop system matrix $A_{cl}$ and its undiscounted counterpart $\hat{A}_{cl}$:
\begin{equation}
    A_{cl} = A - BR^{-1}B^\top P =  A -\frac{1}{2\tau}I - BR^{-1}B^\top P + \frac{1}{2\tau}I  = \hat{A}_{cl} + \frac{1}{2\tau}I.
\end{equation}
Therefore, $A_{cl}  \succ \hat{A}_{cl}$.

By Corollary~\ref{corollary: only one sol make stable} of undiscounted version, there exists at most one $P$ that makes $\hat{A}_{cl} \prec 0$. Therefore, at most one $P$ makes $A_{cl} \prec 0$.
\end{proof}
Note, the stable solution only exists when $\tau$ is sufficient large. In other words, if the future cost is discounted too quickly, the policy becomes myopic and fails to provide any guarantees for long-term behavior.

\paragraph{Proof of Theorem~\ref{thm: solutions to discrete-ARE}}
\begin{proof}
Since $\begin{bmatrix}P_1\\P_2\end{bmatrix}$ are invariant subspace of $H$, following the proof of Theorem~\ref{thm: solution to ARE}, we have 
\begin{equation}
    \begin{bmatrix}
        A + BR^{-1}B^\top A^{-\top}Q & - BR^{-1}B^\top A^{-\top}\\
        -A^{{-\top}}Q &A^{-\top}
    \end{bmatrix} \begin{bmatrix}
        I \\ P_2P_1^{-1}
    \end{bmatrix} = \begin{bmatrix}
        \hat{T} \\
        P_2P_1^{-1}\hat{T}
    \end{bmatrix},
\end{equation}
which gives two set of matrix equations:
\begin{align}
    &\hat{T} = A + BR^{-1}B^\top A^{-\top}Q - BR^{-1}B^\top A^{-T} P_2 P_1^{-1},\\
    &P_2 P_1^{-1} \hat{T} = -A^{-\top} Q + A^{-\top} P_2P_1^{-1}.
\end{align}
Substitute the first set of equations into the second gives
\begin{equation}
    P_2 P_1^{-1} (A + BR^{-1}B^\top A^{-\top}Q - BR^{-1}PB^\top A^{-\top} P_2 P_1^{-1}) = -A^{-\top} Q + A^{-\top} P_2P_1^{-1} \label{eq: discrete sub-invarient intermedium step}.
\end{equation}
Rewriting equation~\eqref{eq: discrete sub-invarient intermedium step} and substituting $P_2P_1^{-1} = P$, we have
\begin{equation}
    PA + (I + P B R^{-1} B^\top) A^{-\top} (Q-P) = 0. \label{eq: discrete sub-invarient intermedium step 2}
\end{equation}
Next, we are going to show $I + P B R^{-1} B^\top$ is invertible.

Suppose there exists non-zero row vector $y$, such that $y (I + P B R^{-1} B^\top) = 0$. Premultiplying~\eqref{eq: discrete sub-invarient intermedium step 2} by $y$ yields $y X = 0$. However, then $y = -y(XBR^{-1}B^\top) = 0$, which contradict the assumption $y \not = 0$. So $I + P B R^{-1} B^\top$ is invertible.
Thus, we have
\begin{equation}
    A^{-\top} (Q-P) = -(I+PBR^{-1}B^\top)^{-1} PA.
\end{equation}
Multiplying~\eqref{eq: discrete sub-invarient intermedium step 2} by $A^\top$, we have
\begin{align}
    A^\top P A + A^\top P BR^{-1}B^\top A^{-\top} (Q-P) + Q - P= 0\\
    Q + A^\top P A - A^\top P BR^{-1}B^\top (I+PBR^{-1}B^\top)^{-1} PA = P \label{eq: discrete almost finish}
\end{align}
Noticing
\begin{align}
(R+B^\top P B) (R^{-1}B^\top) (I+PBR^{-1}B^\top)^{-1} = (B^\top + B^\top PBR^{-1}B) (I+PBR^{-1}B^\top)^{-1} = B^{\top},  
\end{align}
we have 
\begin{equation}
    (R^{-1}B^\top) (I+PBR^{-1}B^\top)^{-1} = (R+B^\top P B)^{-1} B^\top \label{eq: discrete attention}.
\end{equation}
Substituting~\eqref{eq: discrete attention} back to~\eqref{eq: discrete almost finish} yields
\begin{equation}
    P = Q + A^\top P A - A^\top P B (R + B^\top P B)^{-1} B^\top P A,
\end{equation}
which completes the proof.
\end{proof}

\paragraph{Proof of Theorem~\ref{thm: number of eigenvalues compare to unit cycle}}
We first introduce the discrete-time counterpart of Lemma~\ref{lemma: no imaginary eigenvalues}
\begin{lemma}\label{lemma: no eigenvalue in unit cycle}
If the pair $(A,B)$ is controllable and the pair $(Q,A)$ is observable, the Hamiltonian matrix $H$~\eqref{eq: discrete-Hamitonian} has no eigenvalue such that $|\lambda| = 1$. 
\end{lemma}
\begin{proof}
Let
\begin{equation}
    H \begin{bmatrix}
        x\\y
    \end{bmatrix} = \lambda \begin{bmatrix}
        x\\y
    \end{bmatrix},
\end{equation} and $|\lambda| = 1$. Then,
\begin{align}
    Ax + BR^{-1} B^\top A^{-\top} Q x - BR^{-1} B = \lambda x\\
    -A^{-\top} Q x + A^{-\top} y = \lambda y.
\end{align}
Multiplying the second equation by $BR^{-1} B^\top$ and adding it to the first give 
\begin{align}
    A x - \lambda B R^{-1} B y = \lambda x\\
    -Qx + y = \lambda A^\top y.
\end{align}
Multiplying the first by $\bar{\lambda}\bar{y}$ and second by $\bar{x}$ yields
\begin{align}
    \bar{\lambda} \bar{y} Ax = \bar{\lambda} \lambda \bar{y} B R^{-1} B y + \bar{\lambda} \lambda \bar{y} x = -\bar{x} Q x- \bar{y} x.
\end{align}
Thus 
\begin{equation}
    -\bar{\lambda}\lambda \bar{y} B R^{-1} B^\top y - \bar{x} Q x = (\bar{y}y - 1) \bar{y}x = 0.
\end{equation}
Therefore,
\begin{align}
    y^\top B = 0 \\
    Q x = 0
\end{align}
The equation $Cx = 0$ implies $y^T A = \lambda^{-1} y^T$, which, together with $y^T B = 0$, identifies $\lambda^{-1}$ as an uncontrollable eigenvalue of the pair $(A, B)$. On the other hand, $y^T B = 0$ implies $Ax = \lambda x$, and combined with $Cx = 0$, this characterizes $\lambda$ as an unobservable eigenvalue of the pair $(C, A)$.
\end{proof}
We now proof the whole Theorem~\ref{thm: number of eigenvalues compare to unit cycle}:
\begin{proof}
Let $J = \begin{bmatrix}0 &I\\-I &0\end{bmatrix}$, we can easily verify $H^\top J H = J$.

So $H \mathbf{x} = \lambda \mathbf{x}$ implies
\begin{align}
    J H \mathbf{x} = \lambda J \mathbf{x} \\
    J H \mathbf{x} = \lambda H^\top (J H \mathbf{x}) \\
    \frac{1}{\lambda} \mathbf{v} = H^\top \mathbf{v},
\end{align}
where $\mathbf{v} = J H \mathbf{x}$. Thus, $\frac{1}{\lambda}$ is also a eigenvalue of $H$.

Given Lemma~\ref{lemma: no eigenvalue in unit cycle}, there must be n eigenvalues within the unit cycle and n outside.
\end{proof}
\paragraph{Proof of Theorem~\ref{thm: eigenvalues and closed-loop system}}
\begin{proof}
By invariant subspace, we have
\begin{equation}
    \begin{bmatrix}
        A + BR^{-1}B^\top Q & - BR^{-1}B^\top A^{-\top}\\
        -A^{{-\top}}Q &A^{-\top}
    \end{bmatrix} \begin{bmatrix}P_1 \\P_2\end{bmatrix} = \begin{bmatrix}P_1 \\P_2\end{bmatrix} \Lambda_1,
\end{equation}
which yields to two matrix equations:
\begin{align}
    AP_1 + BR^{-1}B^\top Q P_1 - BRB^\top A^{-\top}P_2 = P_1 \Lambda_1 \\
    -A^{-\top} Q P_1 + A^{-\top}P_2 = P_2 \Lambda_1.
\end{align}
Substituting $P_2 = PP_1$ into the second equation
\begin{equation}
   A^{-\top} (P-Q)P_1 = P P_1 \Lambda_2.
\end{equation}
Noticing
\begin{equation}
    P A_{cl} = PA-PB(R+B^\top P B)^{-1} B^\top P A = A^{-\top} (P-Q).
\end{equation}
So,
\begin{equation}
    P A_{cl} P_1 =  A^{-\top} (P-Q)P_1 = P P_1 \Lambda_2. \label{eq: same eigen value proof almost there}
\end{equation}
Substituting $P = P_2P_1^{-1}$ into~\eqref{eq: same eigen value proof almost there}, yields
\begin{equation}
    P_2 P_1^{-1} A_{cl} P_1 = P_2 \Lambda_1.
\end{equation}
Therefore, the diagonal entries of $\Lambda_1$ are eigenvalues of $P_1^{-1} A_{cl} P_1$. Since similarity transformations preserve the eigenvalues of a matrix, $A_{cl}$ also have same set of eigenvalues $\Lambda_1$
\end{proof}
\paragraph{Proof of Theorem~\ref{thm: stable sol with lyapunov conditions}}
Here, we introduce the Lyapunov Theorem and use it to finalize the proof.
\begin{lemma} \label{lemma: lypunov}
\textbf{(Lyapunov Theorem)}
Consider a dynamical system described by the differential equation:
\begin{equation}
    \dot{\mathbf{x}}(t) = f(\mathbf{x}(t)), \quad \mathbf{x}(t) \in \mathbb{R}^n,
\end{equation}
where $ f: \mathbb{R}^n \to \mathbb{R}^n $ is continuously differentiable.

Let $ \mathcal{V}(\mathbf{x}(t)) $ be a continuously differentiable scalar function, called a Lyapunov function, with the following properties: $ \mathcal{V}(\mathbf{x}(t)) > 0 $ for $ \mathbf{x} \neq \mathbf{x}_\text{eq} $ and $ \mathcal{V}(\mathbf{x}_\text{eq}) = 0 $.

Then, if $\dot{\mathcal{V}}(\mathbf{x}(t)) \leq 0$ for all $\mathbf{x}(t) \neq \mathbf{x}_\text{eq}$, then the equilibrium point at $\mathbf{x}_\text{eq}$ is stable,
where the derivative of $ \mathcal{V}(\mathbf{x}(t)) $ along the trajectories of the system is given by:
\begin{equation}
    \dot{\mathcal{V}}(\mathbf{x}(t)) = \frac{\partial \mathcal{V}(\mathbf{x}(t))^\top}{\partial \mathbf{x}} f(\mathbf{x}(t)).
\end{equation}
Furthermore, if $ \mathcal{\dot{V}}(x) < 0 $ for all $ \mathbf{x} \neq \mathbf{x}_\text{eq}$, the equilibrium point at $\mathbf{x}_\text{eq}$ is asymptotically stable and we have $\lim_{s\rightarrow\infty} \|\mathbf{x}(s) - \mathbf{x}_\text{eq}\| = 0$ for any initial condition $\mathbf{x}(t)$.
\end{lemma}
\begin{proof}
See any control textbook or the English translation of Lyapunov's original thesis~\citep{lyapunov1992general}.
\end{proof}
We now prove Theorem~\ref{thm: stable sol with lyapunov conditions} by showing the solution to Bellman equation with positive definite model is a Lyapunov function.
\begin{proof}
Since $\mathcal{V}(\mathbf{x}(t))$ is the solution to HJB equation, with the policy $\mathbf{u}^*(t) = \pi^*(\mathbf{x}(t))$, we have
\begin{equation}
    \dot{\mathcal{V}}(\mathbf{x}(t)) = \frac{\partial \mathcal{V}(\mathbf{x})^\top}{\partial \mathbf{x}} f(\mathbf{x}(t), \pi^*(\mathbf{x}(t))) = - l(\mathbf{x}(t), \pi^*(\mathbf{\mathbf{x}}(t))) \leq 0. 
\end{equation}
Additionally, we have $\mathcal{V}(\mathbf{x}) > 0$ for all $\mathbf{x} \neq \mathbf{x}_\text{eq}$ and $\mathcal{V}(\mathbf{x}_\text{eq})=0$.
Therefore $\mathcal{V}(\mathbf{x}(t))$ is also a Lyapunov function for closed-loop system $\dot{\mathbf{x}} = f(\mathbf{x}, \pi^*(\mathbf{x}))$. By Lemma~\eqref{lemma: lypunov}, the closed-system is stable. Furthermore, if we have $l(\mathbf{x}(t), \pi^*(\mathbf{\mathbf{x}}(t))) > 0$ for all $\mathbf{x} \neq \mathbf{x}_\text{eq}$, the closed-loop is asymptotically stable.
\end{proof}
\begin{remark}
We assume no discount on the running cost in the proof of Theorem~\ref{thm: stable sol with lyapunov conditions}. 
This assumption allows the policy to equally weigh all future consequences of the current control, which is crucial for proving the long-term behavior of the system such as stability.
Additionally, the value function $\mathcal{V}(\mathbf{x}) < \infty$ is well-defined, given the presence of equilibrium points.

However, Theorem~\ref{thm: stable sol with lyapunov conditions} may not hold if the discount factor $e^{-\frac{s}{\tau}} \ll 1$ is small, due to myopic policy behavior. 
In such cases, local stability may still be proven under additional assumptions on the running cost and system dynamics.
\end{remark}

%% file: sections/tabular_setting.tex
\section{Bellman Equation in Tabular Setting} \label{sec: tabular}
Here, we include the proof showing that the value function is the unique solution to the Bellman equation in the tabular setting for comparison. The proof is mainly adapted from~\citet{blackwell1965discounted} and the lecture note from~\citet{gangwani_lecture_2019}.

We assume both $\mathcal{X}$ and $\mathcal{U}$ is finite sets with size $|\mathcal{X}|$ and $|\mathcal{U}|$, respectively. 

For simplicity, we reframe the problem from minimizing the cost-to-go to maximizing the reward-to-go, where the reward is defined as $r(\mathbf{x}, \mathbf{u}) = -l(\mathbf{x}, \mathbf{u})$.

We define Bellman operator $\mathcal{T}: \mathbb{R}^{|\mathcal{X}|} \rightarrow  \mathbb{R}^{|\mathcal{X}|} $ for any $\mathcal{W} \in \mathbb{R}^{|\mathcal{X}|}$ as 
\begin{equation}
    (\mathcal{T} \mathcal{W})(\mathbf{x}) = \underset{\mathbf{u}}{\max} \  [r(\mathbf{x}, \mathbf{u}) + \gamma \mathcal{W}(f(\mathbf{x, \mathbf{u}}))].
\end{equation}
\begin{theorem}
Bellman operator $\mathcal{T}$ is a contraction mapping under sup-norm $\|\cdot\|_\infty$, i.e.,
\begin{equation} \label{thm: contraction mapping}
    \|\mathcal{T} \mathcal{W} - \mathcal{T}\mathcal{Z}\|_\infty \leq \gamma \|\mathcal{W} - \mathcal{Z}\|_\infty, \forall \mathcal{W}, \mathcal{Z} \in \mathbb{R}^{|\mathcal{X}|}.
\end{equation}
\end{theorem}
\begin{proof}
For any particular state $\mathbf{x} \in \mathcal{X}$, we have
\begin{align}
    |(\mathcal{T} \mathcal{W})(\mathbf{x}) - (\mathcal{T}\mathcal{Z})(\mathbf{x})| &= | \underset{\mathbf{u}}{\max} \  [r(\mathbf{x}, \mathbf{u}) + \gamma \mathcal{W}(f(\mathbf{x, \mathbf{u}}))] - \underset{\mathbf{u}}{\max} \  [r(\mathbf{x}, \mathbf{u}) + \gamma \mathcal{Z}(f(\mathbf{x, \mathbf{u}}))]| \\
    &\leq \gamma \underset{\mathbf{u}}{\max} | \mathcal{W}(f(\mathbf{x}, \mathbf{u})) - \mathcal{Z}(f(\mathbf{x}, \mathbf{u}))| \\
    &\leq \gamma \|\mathcal{W} - \mathcal{Z}\|_\infty.
\end{align}
Since, the above holds for any state $\mathbf{x}$, we can conclude
\begin{equation}
    \|\mathcal{T} \mathcal{W} - \mathcal{T}\mathcal{Z}\|_\infty =\underset{\mathbf{x}}{\max}|(\mathcal{T} \mathcal{W})(\mathbf{x}) - (\mathcal{T}\mathcal{Z})(\mathbf{x})| \leq \gamma \|\mathcal{W} - \mathcal{Z}\|_\infty.
\end{equation}
\end{proof}
We can easily verified value function $\mathcal{V}$ is a fixed point to Bellman operator.
Furthermore, value iteration defined as 
\begin{equation}
    \mathcal{W}_{k+1} = \mathcal{T}\mathcal{W}_k, \label{eq: value iteration}
\end{equation}
converge to the value function $\mathcal{V}$.
\begin{theorem} \label{thm: value iteration converge}
Value iteration~\eqref{eq: value iteration} converges to value function $\mathcal{V}$, i.e., 
\begin{equation}
    \lim_{k \rightarrow \infty} \mathcal{W}_k = \mathcal{V},
\end{equation}
where $\mathcal{W}_{k} = \mathcal{T}^{k-1} \mathcal{W}_0$
\end{theorem}
\begin{proof}
Since $\mathcal{V}$ is a fixed point of $\mathcal{T}$ and according to Theorem~\ref{thm: contraction mapping}, we have
\begin{equation}
    \|\mathcal{W}_k - \mathcal{V}\|_\infty = \| \mathcal{T} \mathcal{W}_{k-1} - \mathcal{T} \mathcal{V}\|_\infty \leq \gamma \|\mathcal{W}_{k-1} -\mathcal{V}\|_\infty \leq \cdots \leq \gamma^k \|\mathcal{W}_0 - \mathcal{V}\|_\infty.
\end{equation}
Let $k \rightarrow \infty$ completes the proof.
\end{proof}
\begin{corollary}
Value function $\mathcal{V}$ is the unique solution to Bellman equation.
\end{corollary}
\begin{proof}
Suppose there are other $\mathcal{V}'$ that also satisfying Bellman equation. 
By Theorem~\ref{thm: value iteration converge}, we have 
\begin{equation}
    \|\mathcal{T}^{k} \mathcal{V}' - \mathcal{V}\|_\infty \leq \gamma^k \|\mathcal{V}' - \mathcal{V}\|_\infty \label{eq: norm between two sol}.
\end{equation}
However, since $\mathcal{V}'$ also satisfying Bellman equation, we have
\begin{equation}
     \|\mathcal{T}^{k} \mathcal{V}' - \mathcal{V}\|_\infty = \|\mathcal{V}' - \mathcal{V}\|_\infty,
\end{equation}
which contradicts the equation~\eqref{eq: norm between two sol}.
\end{proof}

%% file: sections/toy_example_with_different_initialization.tex
\section{Effect of Initialization}\label{sec: more details LQR}
Despite the existence of infinitely many solutions for the LQR problem in Section~\ref{sec: canonical failure}, we observe that the neural network tends to favor certain solutions over others.
For example, when using self-normalized initialization methods such as LeCun normal or Kaiming normal~\citep{klambauer2017self}, the network tends to converge to an unstable negative definite solution. 
We hypothesize that this may be due to the negative definite solution having the smallest eigenvalues among the possible solutions, which are closer to the initial weights generated by these self-normalized methods.

\begin{figure}[ht]
    \begin{center}
    \includegraphics[width=0.8\textwidth]{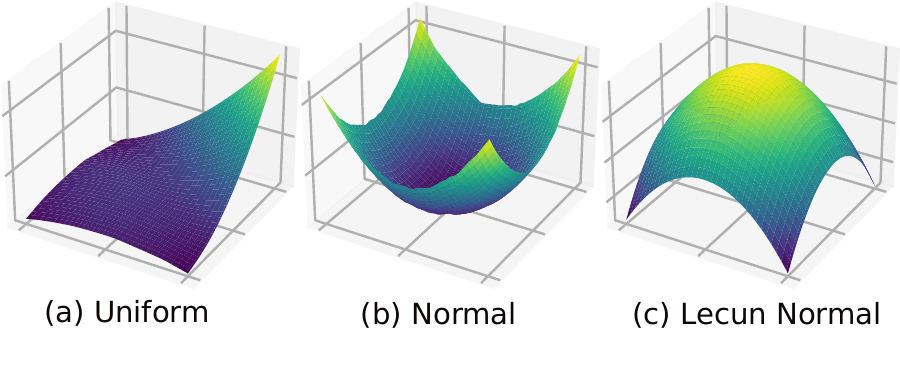}
    \end{center}
    \caption{Solutions found with different initialization method}
    \label{fig: more initialization}
\end{figure}

We further tested different initialization methods, including uniform initialization within $\pm 0.5$ and initialization using a normal distribution with a standard deviation of 1. 
Figure~\ref{fig: more initialization} shows the solutions found by the neural network with different initialization methods.

%% file: sections/experiment_setup.tex
\section{Experiment Setup}\label{sec: experiment setup}

\subsection{Tasks}

\paragraph{Cartpole}
Cartpole is a classical control task for reinforcement learning. 
In this task, we consider bring the cart to the origin while maintain the posture of pole to upright position. The state is denote as $[p, \theta, v, \omega]^\top$, where $p$ denotes the position, $\theta$ is the angle measured from the upright position, and $v$ and $\omega$ denotes velocity and angular velocity, respectively. The dynamics is given by 
\begin{equation}
    \dot{\mathbf{x}}(t) = f(\mathbf{x},\mathbf{u}) = \begin{bmatrix}
    v\\ 
    \omega\\
    \frac{u + m_p\sin\theta (l\dot{\theta}^2 - g\cos\theta)}{m_c + m_p \sin^2 \theta}\\
    \frac{u\cos\theta + m_p l \dot{\theta}^2 \cos\theta \sin\theta - (m_c+m_p)g\sin\theta}{l(m_c + m_p \sin^2\theta)}
    \end{bmatrix},
\end{equation}
where $m_c=1kg$, $m_p=0.1kg$ are mass of cart and pole respectively, and $l=1m$ is the length of rod, $g=9.81m/s^2$ is the gravity constant.
\begin{wrapfigure}{r}{0.5\textwidth}
    \vskip -0pt
    \begin{center}
    \includegraphics[width=0.5\textwidth]{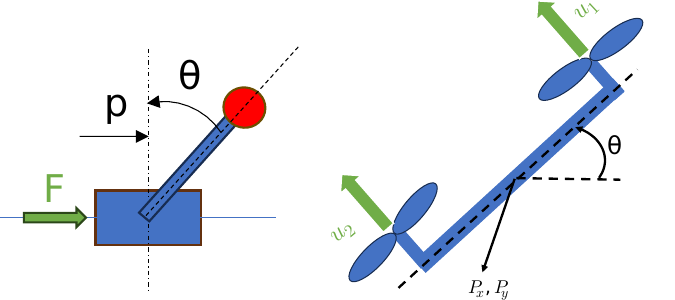}
    \end{center}
    \vskip -5pt
    \caption{System Diagram} 
    \vskip -20pt
    \label{fig:insufficient boundary conditions}
\end{wrapfigure}

Each time, the cartpole is initialized uniformly within the $\pm[4, 0.5, 2, 2]^\top$ and will reset the environment when the state goes beyond $\pm[10, 10, 1000,1000]$
The control input is the force applied to cartpole, with maximum value equal to the weight of the pole. 
The running cost is defined as
\begin{equation}
    l(\mathbf{x}(t), \mathbf{u}(t)) = \|\mathbf{x}(t)\|_2^2 + \|\mathbf{u}(t)\|_2^2
\end{equation}
\paragraph{Drone}
The 2D drone is a canonical underactuated system in the control community. 
The state is define as $[p_x, p_y, \theta, v_x, v_y, \omega]$, where $p_x$, $p_y$ represent the position in the XY plane, $\theta$ denotes the orientation, $v_x$, $v_y$ are linear velocity, and $\omega$ is the angular velocity. The dynamics is given by
\begin{equation}
    \dot{\mathbf{x}}(t) = f(\mathbf{x}, \mathbf{u}) = \begin{bmatrix}
        v_x\\
        v_y\\
        \omega\\
        -\frac{(u_1+u_2) \sin \theta}{m} \\
        \frac{(u_1+u_2)\cos\theta}{m} -g \\
        \frac{r(u_1-u_2)}{I}
    \end{bmatrix},
\end{equation}
where $m=1kg$ is the mass, $r=0.25m$ is the length of the wing, $I=0.0625 kg\cdot m^2$ is the inertia and $g=9.81 m/s^2$ is gravity constant.

The task is to control the drone to reach a desired position and maintain hovering.
Particularly, the drone is initialized uniformly within $\pm[2, 2, 2, 2, 2, 2 ]^\top$ around the target state and is reset if the state exceeds $\pm[4,4,4,5,5,5]$ from the target.
The control $[u_1, u_2]^\top$ represents the thrust generated by the propellers, with a maximum value of twice the total weight of the drone.
The running cost is defined as
\begin{equation}
    l(\mathbf{x}(t), \mathbf{u}(t)) = \|\mathbf{x}(t)\|_2^2 + \|\mathbf{u}(t)\|_2^2.
\end{equation}

\begin{figure}[ht]
    \begin{center}
    \includegraphics[width=\textwidth]{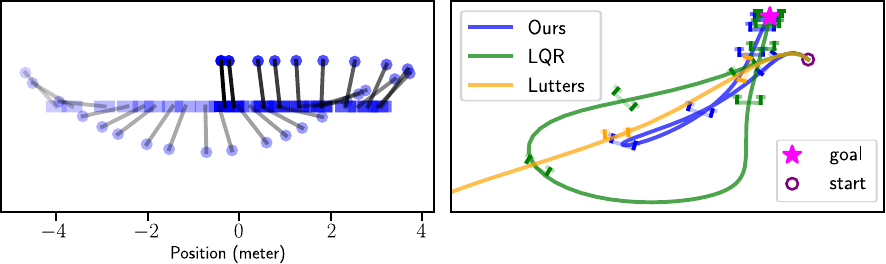}
    \end{center}
    \caption{Example trajectories generated using our method. Left: The cartpole starts with high angular and linear velocity, and our method successfully stabilizes it to the origin. Right: Drone tracking problem. We compare trajectories with~\citet{lutter2020hjb} and a linearized LQR controller. Our method finds the optimal trajectory by effectively leveraging the nonlinear dynamics.}
    \label{fig: example trajectory}
\end{figure}

\subsection{Algorithm}

We minimize the following weighted TD loss:
\begin{equation}
    \mathcal{L}_\theta = \frac{1}{N} \sum_{i=1}^N |1 + \frac{\frac{\partial \mathcal{\hat{V}}_\theta(\mathbf{x}^i)}{\partial \mathbf{x}}^\top f(\mathbf{x}^i, \mathbf{u}^i) - \frac{1}{\tau} \mathcal{\hat{V}_\theta} (\mathbf{x}^i)}{l(\mathbf{x}^i, \mathbf{u}^i)}| \label{eq: weighted TD error}
\end{equation}
where $N$ is the number of total states collected via Monte Carlo sampling.
Compared to the original TD error~\eqref{eq: continuous-time TD}, the loss~\eqref{eq: weighted TD error} weights the error based on the running cost, providing a well-scaled measure that keeps the total loss within a reasonable range.

\begin{algorithm}[h]
\caption{HJB Equation Learning}
\label{alg:vhjb}
\begin{algorithmic}[1]
\STATE Initialize the parameters $\theta$ for $\mathcal{V}_\theta$, empty dataset $\mathcal{D}=\{\}$
\FOR{$i = 1:n_\text{epoch}$}
\STATE Rollout $n_\text{rollout}$ trajectories and append states visited to $\mathcal{D}$
\FOR{batch of $\mathbf{x}^i$ in $\mathcal{D}$}
\STATE Compute the optimal control $\mathbf{u}^i$ based on~\eqref{eq: clipped optimal control}
\STATE Compute the weighed TD error $\mathcal{L}_\theta$ based on~\eqref{eq: weighted TD error}
\STATE Update the value function: $\theta \leftarrow \theta - \alpha \frac{d \mathcal{L}_\theta}{d \theta}$
\ENDFOR
\ENDFOR
\end{algorithmic}
\end{algorithm}

In both tasks, we control an affine system with box-constrained control limits.  Therefore, we compute the control $\mathbf{u}^i$ analytically:
\begin{equation}
    \mathbf{u}^i = \text{clip}(-\frac{1}{2} R^{-1} f_2^\top(\mathbf{x}^i) \frac{\partial \mathcal{V}(\mathbf{x}^i)}{\partial \mathbf{x}}, \mathbf{u}_\text{min}, \mathbf{u}_\text{max}), \label{eq: clipped optimal control}
\end{equation}
where $f(\mathbf{x}^i)$ is the control Jacobian matrix from~\eqref{eq: nonlinear control problem}.

The Algorithm~\ref{alg:vhjb} summarize the main algorithm.

\subsection{Hyperparameters}

We use the same hyperparameters for both tasks and summarized them in Table~\ref{tab:hyperparameters}.
\begin{table}[htbp]
    \caption{Hyperparameters}
    \begin{center}
        \begin{tabular}{ll}
            \multicolumn{1}{l}{\bf Parameter names}  &\multicolumn{1}{l}{\bf Value}
            \\ \hline
            Network size  &[128, 128, 64] \\
            \hline
            Learning rate & $10^{-3}$ \\
            \hline
            Batch size & 256 \\
            \hline
            Number of rollout & 20 \\
            \hline
            Maximum trajectory length &200\\
            \hline
            $\epsilon$ & $10^{-3}$\\
            \hline
            Activation function &ELU\\
            \hline
            Kernel initialization &Lecun normal\\
            \hline
            Optimizer &Adam\\
            \hline
        \end{tabular}
    \end{center}
    \label{tab:hyperparameters}
\end{table}